\begin{document}
\twocolumn[
\icmltitle{Architectural Complexity Measures of Recurrent Neural Networks}

\icmlauthor{Saizheng Zhang$^{1, *}$}{saizheng.zhang@umontreal.ca}
\icmlauthor{Yuhuai Wu$^{2, *}$}{ywu@cs.toronto.edu}
\icmlauthor{Tong Che$^3$}{tongche@ihes.fr}
\icmlauthor{Zhouhan Lin$^1$}{lin.zhouhan@gmail.com}
\icmlauthor{Roland Memisevic$^{1, 4}$}{roland.umontreal@gmail.com}
\icmlauthor{Ruslan Salakhutdinov$^{2, 4}$}{rsalakhu@cs.toronto.edu}
\icmlauthor{Yoshua Bengio$^{1, 4}$}{yoshua.bengio@gmail.com}
\icmladdress{$^1$Université de Montréal, $^2$University of Toronto, $^3$Institut des Hautes Études Scientifiques, France, $^4$CIFAR}
\vspace{-10pt}
\icmladdress{* Equal contribution}

\icmlkeywords{recurrent neural networks, deep learning, deep architectures}

\vskip 0.3in
]
\begin{abstract} 
In this paper, we systematically analyse the connecting architectures
of recurrent neural networks (RNNs). Our main contribution is twofold:
first, we present a rigorous graph-theoretic framework describing the
connecting architectures of RNNs in general.
Second, we propose three 
architecture complexity measures of RNNs: (a) the {\bf recurrent depth}, which
captures the RNN's over-time nonlinear complexity, (b)
the {\bf feedforward depth}, which captures the local input-output nonlinearity (similar
to the ``depth" in feedforward neural networks (FNNs)), 
and (c) the {\bf recurrent skip coefficient}
which captures how rapidly the information
propagates over time. Our experimental results show that RNNs might benefit
from larger recurrent depth and feedforward depth.
We further demonstrate that increasing recurrent skip coefficient
offers performance boosts on long term dependency problems,
as we improve the state-of-the-art for sequential MNIST dataset.

\end{abstract}
\section{Introduction}
\label{sec:intro}
Recurrent neural networks (RNNs) have been shown to achieve promising results 
on many difficult sequential learning problems \cite{graves2013generating,
bahdanau2014neural, sutskever2014sequence}.
There is also much work attempting to
reveal the principles behind the challenges and successes of RNNs,
including optimization issues \cite{martens2011learning, pascanu2013difficulty}, gradient vanishing/exploding related problems
\cite{hochreiter1991untersuchungen,bengio1994learning},
 analysing/designing new RNN transition functional units
like LSTMs, GRUs and their variants \cite{hochreiter1997long,
greff2015lstm, cho2014learning, jozefowicz2015empirical}. 

This paper focuses on another important theoretical aspect of RNNs:
the connecting architecture. Ever since
\citet{schmidhuber1992learning, el1996hierarchical}
introduced  different forms of ``stacked RNNs'', researchers
have taken architecture design for granted and have paid less
attention to the exploration of other connecting architectures.
Only a few researches has been done alone this direction: 
\citet{raiko2012deep, graves2013generating}
explored the use of skip connections; \citet{hermans2013training} proposed 
the deep RNNs which are stacked RNNs with skip connections;
 \citet{pascanu2013construct} points out the distinction of
 constructing ``deep" RNN from the view of the recurrent
paths and the view of the input-to-hidden and hidden-to-output maps.
However, they did not rigorously formalize the notion of
``depth'' and its implications in ``deep'' RNNs. 
Besides ``deep" RNN, there still remains a vastly
unexplored field of connecting architectures.
We argue that one barrier for
fully understanding the architectural complexity
is the lack of a general definition of the connecting architecture. This forced previous researchers 
to mostly consider the simple cases while
neglecting other possible connecting variations. Another barrier is the lack
of quantitative measurements of the complexity of different RNN connecting
architectures: even the concept of ``depth'' is not clear
with current RNNs.

In this paper, we try to address these two barriers.
We first introduce a general definition of RNN, where we divide an 
RNN into two basic ingredients: a well-defined graph representation of
the connecting architecture, and a set of transition functions describing the computational process associated with each unit in the network.
Observing that the RNN undergoes multiple transformations not only feedforwardly
(from input to output within a time step) but also recurrently
(across over multiple time steps),  we carry out a quantitative analysis of
the number of transformations in these two orthogonal directions,
which results in the definitions of \emph{recurrent depth}
and \emph{feedforward depth}.
These two depths can be viewed as general extensions of the work of 
\citet{pascanu2013construct}.
In addition, we also explore a quantity called the \emph{recurrent skip coefficient}
which measures how quickly information propagates over time. This quantity
is strongly related to vanishing/exploding gradient issues, and helps deal with long term dependency problems. 
Related work along this line includes \citet{Lin-ieeetnn96, el1996hierarchical, sutskever2010temporal, koutnik2014clockwork}
in which they proposed connections crossing different timescales.
Instead of specific architecture design, we focus on analyzing
the graph-theoretic properties of recurrent skip coefficients, revealing the fundamental difference
between the regular skip connections and the ones which truly increase the recurrent skip coefficients.

We empirically evaluate models with
different recurrent depth, feedforward depth and
recurrent skip coefficients on language modelling and
sequential MNIST tasks. We further demonstrate
the usefulness of the proposed definitions from experimental results. 

\section{General RNN}
RNNs are learning machines which recursively compute new
states by applying transition functions to previous states and inputs.
Such a learning machine has two ingredients: the connecting architecture
describing how information flows between different nodes and
the transition functions describing the nonlinear transformation at each node. 
The connecting architecture is usually illustrated informally by an
infinite directed acyclic graph, which in turn can be viewed as a finite directed
cyclic graph that is unfolded through time. Such informality typically limits one
to consider simple connecting architectures as mentioned in Section \ref{sec:intro}.
In this section, we first introduce a general definition of the connecting
architecture and its underlying computation, followed by a general definition of RNNs.


\subsection{The Connecting Architecture}
\label{sec:connect_arch}
We formalize the concept of the connecting architecture by extending the traditional graph-based illustration 
to a more general definition with a \emph{finite directed multigraph} and its \emph{unfolded} version. 
For this purpose, we first define the concept of
\emph{RNN cyclic graphs}, which is the formalization of conventional cyclic graphical representation of RNNs.
We attach ``weights'' to edges in the cyclic graph $\mathcal{G}_c$. These weights correspond to the time delay difference between the
source and destination node in the unfolded graph representation.
\begin{defn}
\label{def:cyc}
\emph{
Let $\mathcal{G}_c = (V_c, E_c)$ be a weighted directed multigraph
\footnote{A directed multigraph is a directed graph
that allows multiple directed edges connecting two nodes.},
in which $V_c = V_{\tmop{in}} \cup V_{\tmop{out}} \cup V_{\tmop{hid}}$ is a finite
set of nodes, and $V_{\tmop{in}}, V_{\tmop{out}}, V_{\tmop{hid}}$ are not empty. $E_c \subset V_c \times V_c \times \mathbb{Z}$ is a finite set of directed edges.
Each $e =(u,v,\sigma) \in E_c$ denotes a directed weighted edge pointing from node $u$ to node $v$ with an integer weight $\sigma$. Each node $v \in V_c$ is \ labelled by an integer tuple $(i, p)$. $i \in \{ 0, 2, \cdots m-1 \}$ denotes the time index of the given node, where $m$ is the \textbf{period number} of the RNN, and $p
\in S$, $S$ is a finite set of node labels.
We call the
weighted directed multigraph $\mathcal{G}_c = (V_c, E_c)$ an RNN cyclic graph, if
}

\emph{
(1) For every edge $e =(u,v,\sigma) \in E_c$, let $i_u$ and $i_v$ denote the time index of node $u$ and $v$, then $\sigma = i_v-i_u+k\cdot m$ for some $k \in \mathbb{Z}$.
}

\emph{
(2) There exists at least one directed cycle
\footnote{A directed cycle is a closed walk with no
repetitions of edges.}
in $\mathcal{G}_c$. 
}

\emph{
(3) For any closed walk $\omega$, the sum of all the $\sigma$ along $\omega$ is not zero.
}

\emph{
(4)  There are no incoming edges to nodes in $V_{\tmop{in}}$, and
no outgoing edges from nodes in $V_{\tmop{out}}$. There are both incoming edges and outgoing edges for nodes in $V_{\tmop{hid}}$.
}
\end{defn}

\begin{figure*}[t]
\includegraphics[width=\textwidth]{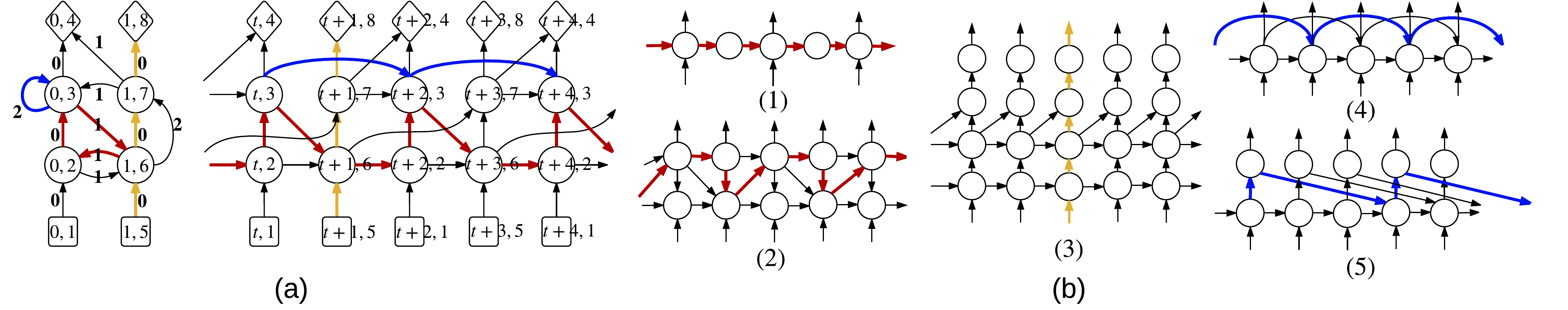}
\vspace{-25pt}
\caption{(a) An example of an RNN's $\mathcal{G}_c$
and $\mathcal{G}_{\tmop{un}}$. $V_{\tmop{in}}$ is denoted
by square, $V_{\tmop{hid}}$ is denoted by circle and 
$V_{\tmop{out}}$ is denoted by diamond. In $\mathcal{G}_c$,
the number on each edge is its corresponding $\sigma$. The longest path is colored in red. The longest input-output path 
is colored in yellow and the shortest path is colored blue. The value of three measures are $d_r = \frac{3}{2}$, $d_f = 3$ and $s = 2$. (b) 5 more examples. (1) and (2) have
$d_r = 2, \frac{3}{2}$, (3) has $d_f = 5$, (4) and (5) has $s = 2, \frac{3}{2}$.
}
\label{fig:basic}
\vspace{-10pt}
\end{figure*}

Condition (1) assures that we can get a periodic graph (repeating pattern) when unfolding the RNN through time.
Condition (2) excludes feedforward neural networks in the definition by forcing to have at least one cycle
in the cyclic graph.
Condition (3) simply avoids cycles after unfolding.

The cyclic representation can be seen as a time folded representation of RNNs,
as shown in Figure \ref{fig:basic} (a). Given an RNN cyclic graph $\mathcal{G}_c$, we unfold $\mathcal{G}_c$ over time $t\in \mathbb{Z}$
by the following procedure:

\begin{defn}[\bf Unfolding]
\label{def:un}
\emph{
Given an RNN cyclic graph \ $\mathcal{G}_c = (V_c, E_c, \sigma)$, we
define a new infinite set of nodes $V_{un} = \{ (i + k m, p) | (i, p) \in V, k
\in \mathbb{Z} \} \nobracket$. The new set of edges $E_{un} \in V_{un} \times
V_{un}$ is constructed as follows: \ $((t, p), (t', p')) \in E_{un}$ if and
only if there is an edge $e = ((i, p), (i', p'), \sigma) \in E$ such that $t' -
t = \sigma$, and $t \equiv i (\tmop{mod}
m)$. The new directed graph $\mathcal{G}_{\tmop{un}} = (V_{un}, E_{un})$ is called the unfolding of $\mathcal{G}_c$. Any infinite directed graph that can be constructed from an RNN cyclic graph through unfolding is called an RNN unfolded graph.
}
\end{defn}

\begin{lemma}
\label{prop:dag}
\emph{
The unfolding $\mathcal{G}_{\tmop{un}}$ of
any RNN cyclic graph $\mathcal{G}_c$ is a directed acyclic graph (DAG).
}
\end{lemma}

Figure \ref{fig:basic}(a) shows an example of two graph representations $\mathcal{G}_{\tmop{un}}$ and $\mathcal{G}_c$ of a given RNN. 
Consider the edge from node $(1,7)$ going to node $(0,3)$ in $\mathcal{G}_c$. The fact that it has weight 1 indicates that the corresponding edge in $\mathcal{G}_{\tmop{un}}$ travels one time step, $((t+1,7),(t+2,3))$. Note that node $(0,3)$ also has a loop with weight 2. This loop corresponds to the edge $((t,3),(t+2,3))$. 

The two kinds of graph representations we presented above have a one-to-one correspondence.
 Also, any graph structure $\theta$ on $\mathcal{G}_{\tmop{un}}$ 
 is naturally mapped into a graph structure $\bar{\theta}$ on $\mathcal{G}_c$. 
 Given an edge tuple $\bar{e}=(u,v,\sigma)$ in $\mathcal{G}_c$, $\sigma$ stands for the number of time steps crossed by $\bar{e}$'s covering edges in $E_{un}$, i.e., for every corresponding edges $e\in \mathcal{G}_{\tmop{un}}$, $e$ must start from some time index $t$ to $t+\sigma$.
$\sigma$ thus corresponds to the ``time delay'' associated with $e$. 

In addition, the \emph{period number}
$m$ in Definition \ref{def:cyc}
can be interpreted as the time length
of the entire non-repeated recurrent structure
in its unfolded RNN graph $\mathcal{G}_{\tmop{un}}$. Strictly speaking, 
$m$ has the following properties in $\mathcal{G}_{\tmop{un}}$:
$\forall k \in \mathbb{Z}$,
if $\exists v = (t, p) \in V_{un}$,
then $\exists v' = (t + km, p) \in V_{un}$;
if $\exists e = ((t, p), (t', p')) \in E_{un}$,
then $\exists e' = ((t+km, p), (t'+km, p')) \in E_{un}$.	
In other words, shifting the $\mathcal{G}_{\tmop{un}}$
through time by $km$ time steps will result in a 
DAG which is identical to $\mathcal{G}_{\tmop{un}}$,
and $m$ is the smallest number that has such property for $\mathcal{G}_{\tmop{un}}$.
Most traditional RNNs have $m = 1$, while some special
structures like {\em hierarchical} or {\em clockwork RNN}~\citep{el1996hierarchical, koutnik2014clockwork}
have $m>1$.
For example, in Figure \ref{fig:basic}(a),
from the unfolded graph representation $\mathcal{G}_{\tmop{un}}$, we can easily  
see that the period number of this specific RNN is 2. 

It is clear that if there exists 
a directed cycle $\vartheta$ in $\mathcal{G}_c$, and
the sum of $\sigma$
along $\vartheta$ is positive (or negative),
then there is a path which is the pre-image of $\vartheta$ in $\mathcal{G}_{\tmop{un}}$
whose length (summing the edge $\sigma$'s) approaches $+\infty$ (or $-\infty$).
This fact naturally induces the general definition of 
unidirectionality and bidirectionality of RNNs as follows:
\begin{defn}
\emph{
An RNN is called unidirectional if its cyclic
graph representation $\mathcal{G}_c$ has the
property that the sums of $\sigma$ along all the
directed cycles $\vartheta$ in $\mathcal{G}_c$
share the same sign, i.e., either all positive
or all negative. An RNN is called bidirectional
if it is not unidirectional.
}
\end{defn}

\subsection{A General Definition of RNN}
 The connecting architecture in Sec.~\ref{sec:connect_arch} describes how information flows 
 among RNN units. Assume $\bar{v} \in V_c$ is a node in $\mathcal{G}_c$, let $\mathrm{In}(\bar{v})$ denotes the set of incoming nodes of $\bar{v}$, $\mathrm{In}(\bar{v}) = \{\bar{u} | (\bar{u}, \bar{v}) \in E_{c}\}$. In the forward pass of the RNN, the transition function $F_{\bar{v}}$ takes outputs of nodes $\mathrm{In}(\bar{v})$ as inputs and computes a new output. For example, vanilla RNNs units with different activation functions, LSTMs and GRUs can all be viewed as units with specific transition functions. 


We now give the general definition of an RNN:
 \begin{defn}
 \em
 An RNN is a tuple $ (\mathcal{G}_{c},\mathcal{G}_{\tmop{un}}, \{F_{\bar{v}}\}_{\bar{v}\in V_{c}})$, 
 in which $\mathcal{G}_{\tmop{un}} = (V_{un}, E_{un})$ is the unfolding of RNN cyclic graph $\mathcal{G}_c$, and $\{F_{\bar{v}}\}_{\bar{v}\in V_{c}}$ is the set of transition functions. In the forward pass, for each $v\in V_{un}$, the transition function $F_{\bar{v}}$ takes all incoming nodes of $v$ as the input to compute the output.
 \end{defn}
  An RNN is \emph{homogeneous} if all the nodes share the same form of transition function (but with perhaps different weight parameters).
\section{Measures of Architectural Complexity}
In this section, we develop different measures of RNNs' architectural complexity, focusing mostly on the graph-theoretic properties of RNNs.
To analyze an RNN solely from its architectural aspect, we make the mild assumption that the RNN is homogeneous. 
We further assume the RNN to be unidirectional. For a bidirectional RNN, it is more natural to measure the complexities of its unidirectional components.

\subsection{Recurrent Depth}
Unlike feedforward models where computations are done within one time frame, RNNs map inputs to outputs over multiple time steps. In some sense, an RNN undergoes transformations along both feedforward and recurrent dimensions. This fact suggests that we should investigate its architectural complexity from these two different perspectives. We first consider the recurrent perspective.

The conventional definition of depth is the \emph{maximum} number of nonlinear transformations from inputs to outputs~\citep{Bengio-2009-book}. 
Observe that a directed path in an unfolded graph representation $G_{un}$ corresponds to a sequence of nonlinear transformations. Given an unfolded RNN graph $G_{un}$, $\forall i,n \in \mathbb{Z}$, let $\mathfrak{D}_i(n)$ be
the length of the \emph{longest} path
from any node at starting time $i$ to any node at time $i+n$.

From the recurrent perspective, it is natural to investigate how $\mathfrak{D}_i(n)$ changes over time. Generally speaking, $\mathfrak{D}_i(n)$ increases as $n$ increases for all $i$. Such
increase is caused by the recurrent structure of the RNN
which keeps adding new nonlinearities over time. 
Since $\mathfrak{D}_i(n)$ approaches $\infty$ as $n$ approaches $\infty$,\footnote{Without loss of generality, we assume the unidirectional RNN approaches positive infinity.} to measure the complexity of $\mathfrak{D}_i(n)$, we consider its asymptotic behaviour, i.e., the limit of $\frac{\mathfrak{D}_i(n)}{n}$ as $n\to \infty$. Under a mild assumption, this limit exists. To perform a practical calculation of this limit, 
the next theorem relies on $\mathcal{G}_{\tmop{un}}$'s cyclic counterpart $\mathcal{G}_{c}$, where the computation is much easier:

\begin{thm}[\bf Recurrent Depth]
\label{thm:main-main}
Given an RNN and its two graph representation $\mathcal{G}_{\tmop{un}}$ and $\mathcal{G}_c$, we denote $C (\mathcal{G}_c)$ to be the set of directed
  cycles in $\mathcal{G}_c$. For $\vartheta$ $\in C (\mathcal{G}_c)$, let  $l
  (\vartheta)$ denote the length of $\vartheta$ and $\sigma_s (\vartheta)$ denote the sum of edge weights
  $\sigma$ along $\vartheta$. Under a mild assumption\footnote{See a full treatment of the limit in general cases in Theorem \ref{thm:main} and Proposition \ref{prop:assump} of Appendix.},
  \begin{equation}
        \label{eqn:main1}
        d_r = \lim_{n \rightarrow +\infty}\frac{\mathfrak{D}_i(n)}{n} = \max_{\vartheta \in C (\mathcal{G}_c)} \frac{l (\vartheta)}{
       \sigma_s (\vartheta)}
    \end{equation}
Thus, $d_r$ is a positive rational number.
\end{thm} 
More intuitively, $d_r$ is a measure of the average maximum number of nonlinear transformations per time step as $n$ gets large. Thus, we call it \emph{recurrent depth}:
\begin{defn}[\textbf{Recurrent Depth}]
\emph{
Given an RNN and its two graph representations $\mathcal{G}_{\tmop{un}}$ and $\mathcal{G}_c$,
we call $d_r$, defined in Eq.(\ref{eqn:main1}), the recurrent depth of the RNN.}
\end{defn}

In Figure \ref{fig:basic}(a), one can easily verify that $\mathfrak{D}_t(1)=5$, $\mathfrak{D}_t(2)=6$, $\mathfrak{D}_t(3)=8$, $\mathfrak{D}_t(4)=9$ $\dots$ Thus $\frac{\mathfrak{D}_t(1)}{1} = 5$, $\frac{\mathfrak{D}_t(2)}{2} = 3$, $\frac{\mathfrak{D}_t(3)}{3} = \frac{8}{3}$, $\frac{\mathfrak{D}_t(4)}{4} = \frac{9}{4}$ $\dots$., which eventually converges to $\frac{3}{2}$ as $n \to \infty$.
As $n$ increases, most parts of the longest path coincides with the path colored in red. As a result, $d_r$ coincides with the number of nodes the red path goes through per time step. Similarly in $\mathcal{G}_c$, observe that the red cycle achieves the maximum ($\frac{3}{2}$) in Eq.(\ref{eqn:main1}). Usually, one can directly calculate $d_r$ from $\mathcal{G}_{\tmop{un}}$. We introduce the computation in $\mathcal{G}_c$ for general purposes. 

It is easy to verify that \emph{simple RNNs} and \emph{stacked RNNs} share the same recurrent depth which equals to 1. This reveals the fact that their nonlinearities increase at the same rate, which suggests they will behave more similarly in the long run. This fact is often neglected, since one would 
typically consider the number of layers as a measure of depth, and think of stacked RNNs as ``deep'' and simple RNNs as ``shallow'', even though their discrepancies are not with respect with
recurrent depth (which regards time) but with feedforward depth, defined next.

\subsection{Feedforward Depth}
Recurrent depth does not fully characterize the nature of nonlinearity of an RNN. 
As previous work suggests \citep{sutskever2014sequence}, stacked
RNNs do outperform shallow ones with the same hidden size, in problems where a more immediate
input and output process is modeled. This is not surprising, since the growth rate of $\mathfrak{D}_i(n)$ only captures the number of nonlinear transformations in the time direction, not in the feedforward direction.

The perspective of feedforward computation puts more emphasis on the specific paths connecting inputs to outputs. Given an RNN unfolded graph $G_{un}$, let $\mathfrak{D}^*_i(n)$ be the length of the longest path from any input node at time step $i$ to any output node at time step $i+n$. Clearly, when $n$ is small, the recurrent depth cannot serve as a good description for $\mathfrak{D}^*_i(n)$. In fact. it heavily depends on another quantity which we call \emph{feedforward depth}. The following proposition guarantees
the existence of such a quantity and demonstrates the role of both measures in quantifying the nonlinearity of an RNN. 


\begin{proposition}[\textbf{Input-Output Length Least Upper Bound}]
\label{prop:sup}
\em
Given an RNN with recurrent depth $d_r$, we denote 
\begin{equation}
\label{eq:ffdepth}
d_f = \sup_{i,n \in \mathbb{Z}}{\mathfrak{D}^*_i(n)} - n\cdot d_r.
\end{equation}
The supremum $d_f$ exists and thus we have the following upper bound for $\mathfrak{D}^*_i(n)$
\begin{equation*}
\mathfrak{D}^*_i(n)  \leq n\cdot d_r +d_f
\end{equation*}
\end{proposition}

The above upper bound explicitly demonstrates the interplay between recurrent depth and feedforward depth: when $n$ is small, $\mathfrak{D}^*_i(n)$ is largely bounded by $d_f$; when $n$ is large, $d_r$ captures the nature of the bound ($\approx{n\cdot d_r}$). These two measures are equally important, as they separately capture the maximum number of nonlinear transformations of an RNN in the long run and in the short run. 

\begin{defn}
(\textbf{Feedforward Depth})\emph{
Given an RNN with recurrent depth $d_r$ and its two graph representations $\mathcal{G}_{\tmop{un}}$ and $\mathcal{G}_c$, we call $d_f$, defined in Eq.(\ref{eq:ffdepth}), the feedforward depth\footnote{Conventionally, an architecture with depth 1 is a three-layer architecture containing one hidden layer. But in our definition, since it goes through two transformations, we count the depth as 2 instead of 1. This should be particularly noted with the concept of feedforward depth, which can be thought as the conventional depth plus 1.} of the RNN.}
\end{defn}
To calculate $d_f$ in practice, we introduce the following theorem:
\begin{thm}[\textbf{Feedforward Depth}]
\label{thm:feed}
Given an RNN and its two graph representations $\mathcal{G}_{\tmop{un}}$ and
  $\mathcal{G}_c$, we denote $\xi (\mathcal{G}_c)$ the set of directed paths 
  that start at an input node and end at an output node in $\mathcal{G}_c$.
  For $\gamma \in \xi(\mathcal{G}_c)$, denote $l (\gamma)$ the length and
  $\sigma_s (\gamma)$ the sum of $\sigma$ along $\gamma$. Then we have:
  \begin{equation*}
      d_f = \sup_{i, n \in \mathbb{Z}} \mathfrak{D}_i^{\ast} (n) - n \cdot d_r
     = \max_{\gamma \in \xi (\mathcal{G}_c)} l (\gamma) - \sigma_s (\gamma)
     \cdot d_r,
  \end{equation*}
where $m$ is the period number and $d_r$ is the recurrent depth of the RNN. Thus, $d_f$ is a postive rational number.
\end{thm}

For example, in Figure \ref{fig:basic}(a), one can easily verify that
$d_f=\mathfrak{D}^*_t(0)=3$. Most commonly, $d_f$ is the same as 
$ \mathfrak{D}^*_t(0)$, i.e., the maximum length from an input to its
current output. 

\subsection{Recurrent Skip Coefficient}
Depth provides a measure of the complexity of the model.
But such a measure is not sufficient to characterize behavior on long-term dependency
tasks. In particular, since models with large recurrent depth have more nonlinearities through time, gradients can explode or vanish more easily. 
On the other hand, it is known 
that adding skip connections across multiple time steps
may help improve the performance on long-term dependency problems 
(\citet{Lin-ieeetnn96, sutskever2010temporal}).
To measure such a ``skipping'' effect, we should instead pay attention to the 
length of the {\em shortest path} from time $i$ to time $i+n$. 
In $G_{un}$, $\forall i,n \in \mathbb{Z}$,
let $\mathfrak{d}_i(n)$ be the length of this shortest path.  
Similar to the recurrent depth, we consider the growth rate of $\mathfrak{d}_i(n)$.

\begin{thm}[\textbf{Recurrent Skip Coefficient}]
\label{thm:rsc}
Given an RNN and its two graph representations $\mathcal{G}_{\tmop{un}}$ and $\mathcal{G}_c$, under mild assumptions\footnote{See Proposition \ref{prop:assump2} in Appendix.}
\begin{equation}
\label{eqn:rsc}
j = \lim_{n \rightarrow +\infty}\frac{\mathfrak{d}_i(n)}{n} = \min_{\vartheta \in C (\mathcal{G}_c)} \frac{l (\vartheta)}{
       \sigma_s (\vartheta)},
\end{equation}
Thus, $j$ is a positive rational number.
\end{thm}

Since it is often the case that $j$ is smaller or equal to 1, it is more straightforward to consider its reciprocal.
\begin{defn}(\textbf{Recurrent Skip Coefficient})\footnote{One would find this definition very similar to the definition of the recurrent depth. Therefore, we refer readers to examples in Figure \ref{fig:basic} for similar illustration.}.
\emph{
Given an RNN and its two graph representations $\mathcal{G}_{\tmop{un}}$ and $\mathcal{G}_c$, we define $s= \frac{1}{j}$, whose reciprocal is defined in Eq.(\ref{eqn:rsc}), as the recurrent skip coefficient of the RNN.}
\end{defn}

With a larger recurrent skip coefficient, the number of transformations per time step is smaller. As a result, the nodes in the RNN are more capable of ``skipping'' across the network, allowing unimpeded information flow across multiple time steps, thus alleviating the problem of
learning long term dependencies. In particular, such effect is more prominent in the long run, due to the network's recurrent structure. 
Also note that not all types of skip connections can increase the recurrent skip coefficient. We will consider specific examples in our experimental results section.




\section{Experiments and Results}
In this section we conduct a series of experiments to investigate the following
questions: (1) Is recurrent depth a trivial measure? (2) Can increasing depth yield performance improvements?
(3) Can increasing the recurrent skip coefficient improve the performance on long term dependency tasks? And (4)
Can the recurrent skip coefficient suggest something more compared to simply adding skip
connections? We first show evaluations on RNNs with
\textit{tanh} nonlinearities, and then present similar results for LSTMs. 


\subsection{Tasks and Training Settings}

\textbf{text8 dataset}: We evaluate our models on character level language modeling
using the text8 dataset \footnote{http://mattmahoney.net/dc/textdata.},
which contains $100M$ characters from Wikipedia 
with an alphabet of letters \textit{a-z} and a space symbol.
We follow the setting from \citet{mikolov2012subword}: $90M$ for training, $5M$ for validation 
and the rest $5M$ for test. We also divide the training set 
into non-overlapping sequences, each of length $180$.
Quality of fit is evaluated by the bits-per-character (BPC) metric,
which is $\log_2$ of perplexity.

\textbf{sequential MNIST dataset}:
We also evaluate our models with a modified version of the MNIST dataset. 
Each MNIST image data is reshaped into a $784\times1$ sequence,
turning the digit classification task into a sequence classification one
with long-term dependencies~\citep{le2015simple,arjovsky2015unitary}.
The model has access to $1$ pixel per time step and
predicts the class label at the end.
A slight modification of the dataset is to permute
the image sequences by a fixed random order beforehand
(permuted MNIST).
Since spatially local dependencies are no longer local
(in the sequence), this task may be harder in terms of modelling long-term dependencies.
Results in \cite{le2015simple} have shown that both \textit{tanh} RNNs and LSTMs  
did not achieve satisfying performance, which also highlights the 
difficulty of this task. 

For each experiment, 
we use \textit{Adam}~\citep{kingma2014adam} for optimization, and conduct 
a grid search
on the learning rate in $\{10^{-2}, 10^{-3},10^{-4},10^{-5}\}$.
For $tanh$ RNNs, 
the parameters are initialized with samples from a uniform distribution.
For the LSTM networks we adopt a similar initialization scheme, while
the forget gate biases are chosen by the grid search on $\{-5, -3, -1, 0, 1, 3, 5\}$.
We employ early stopping during training, and the batch size is set to $50$.

\subsection{\bf Recurrent Depth is Non-trivial}
\label{sec:depth_nontrivial}
To investigate the first question, we compare 4 similar connecting architectures: 
1-layer (shallow) ``$sh$", 2-layers stacked ``$st$",
2-layers stacked with an extra bottom-up connection ``$bu$", and
2-layers stacked with an extra top-down connection ``$td$",
see Figure \ref{fig:depth_graph} (a). 
\textit{sh} has a layer size of $512$, while the rest of the networks 
have a layer size of $256$, making the number of hidden parameters
in each architecture remain roughly the same.
We also compare our results to the ones reported in~\cite{pascanu2012understanding} 
(denoted as ``P-$sh$"), since they used the 
same model architecture as the one specified by our $sh$.  

Although the four architectures look quite similar, they have 
different recurrent depths: \textit{sh}, \textit{st} and \textit{bu} have
$d_r = 1$ but \textit{td} has $d_r = 2$. Note that the specific 
construction of the extra nonlinear transformations in 
\textit{td} is not conventional.
Instead of simply adding intermediate layers in hidden-to-hidden
connection, as reported in~\cite{pascanu2013construct}, more nonlinearities are gained by a 
recurrent flow from the first layer to the second layer and then back 
at each time step (see the red path in Figure \ref{fig:depth_graph}).

\begin{figure}[t]
\vspace{-8pt}
\includegraphics[width=\columnwidth]{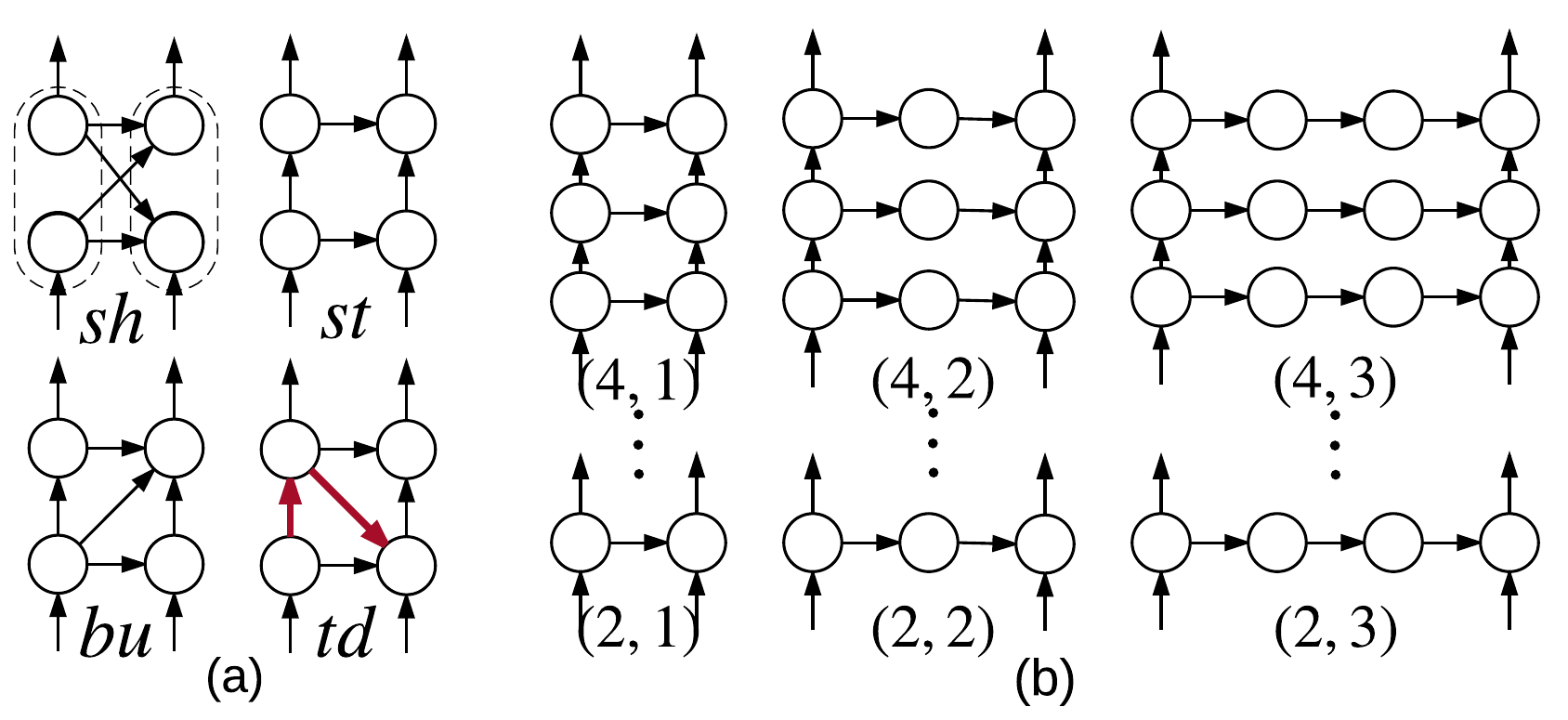}
\vspace{-22pt}
\caption{(a) the architectures for $sh$, $st$, $bu$ and $td$, with their
$(d_r, d_f)$ equals to $(1, 2)$, $(1, 3)$, $(1, 3)$ and $(2, 3)$, respectively.
The longest path in $td$ are colored in red.
(b) the 9 architectures denoted by their $(d_f, d_r)$ that
 $d_r = 1, 2, 3$ and $d_f = 2, 3, 4$.
We only plot the hidden states within 1 time steps (which is also 1 period) in
both (a) and (b).}
\label{fig:depth_graph}
\vspace{-10pt}
\end{figure}

\begin{table}[h]
\vspace{-5pt}
\caption{test BPCs of \textit{sh}, \textit{st}, \textit{bu}, \textit{td} for $tanh$}
\label{tb:butd}
\begin{center}
\begin{small}
\begin{sc}
\begin{tabular}{c|ccc|c|c}
\hline
\hline
test BPC &\textit{sh} & \textit{st} & \textit{bu} & \textit{td} & P-$sh$\\
\hline
$text8$& $1.80$ & $1.82$ & $1.80$ & $\mathbf{1.77}$ & $1.80$\\
\hline
\end{tabular}
\end{sc}
\end{small}
\end{center}
\vspace{-10pt}
\end{table}

Table \ref{tb:butd} clearly shows that the \textit{td} architecture outperforms the others,
including the results reported in~\cite{pascanu2012understanding}.
The validation curve, displayed in Figure \ref{fig:4architec}(a),
further shows the gap between
\textit{td} and the other three architectures.
Even though \textit{td}'s validation BPC is higher
at the beginning of the training (possibly due to harder optimization),
it outperforms all other architectures in the long run.

It is also interesting to note the 
improvement we obtain when switching from \textit{bu} to \textit{td}.  
Note that the only difference between these two architectures 
lies in changing the direction of one connection (see Figure \ref{fig:depth_graph}(a)), 
although this also increases
the recurrent depth. Such a fundamental difference is by no means
self-evident, 
but this result highlights the necessity of the concept of
recurrent depth.
\begin{figure}[t]
\vspace{-5pt}
\centering
\includegraphics[width=170pt]{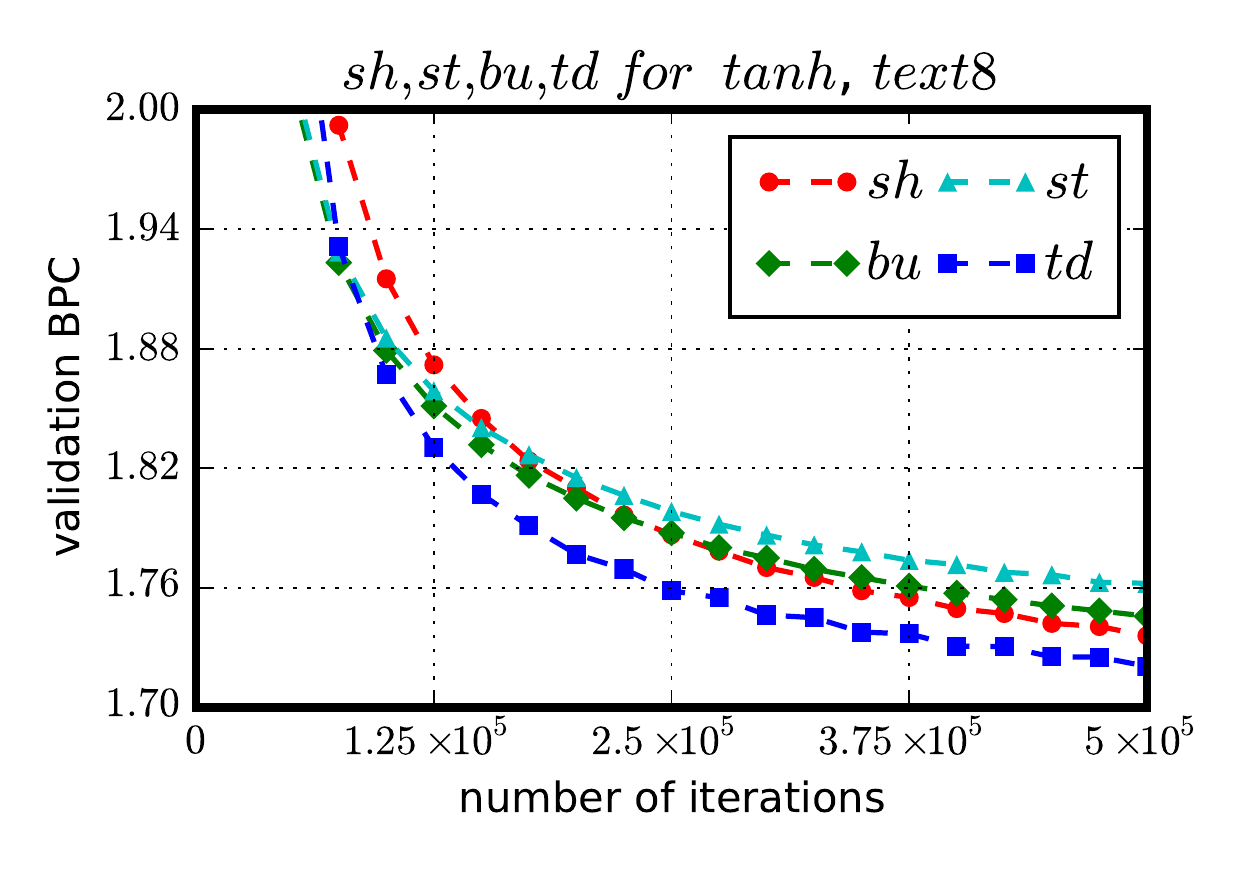}
\vspace{-20pt}
\caption{Validation curves of $sh$, $st$, $bu$ and $td$ for $tanh$, $text8$.}
\label{fig:4architec}
\vspace{-10pt}
\end{figure}

\begin{figure*}[t]
\includegraphics[width=\textwidth]{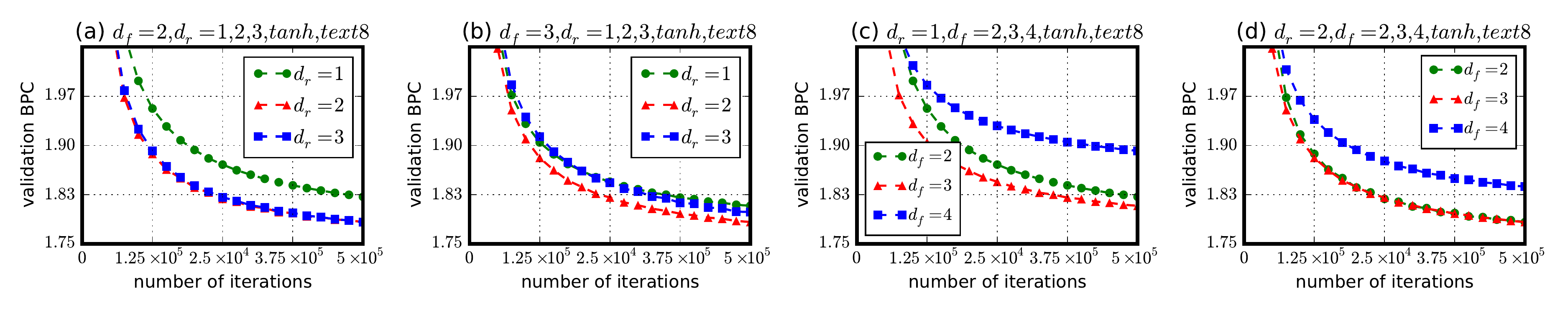}
\vspace{-30pt}
\caption{Validation curves for architectures with different $d_r$ and $d_f$.
(a) Fix $d_f = 2$, change $d_r$ from $1$ to $3$.
(b) Fix $d_f = 3$, change $d_r$ from $1$ to $3$.
(c) Fix $d_r = 1$, change $d_r$ from $2$ to $4$.
(d) Fix $d_r = 2$, change $d_r$ from $2$ to $4$.}
\label{fig:depth_curve}
\vspace{-10pt}
\end{figure*}

\subsection{Full Comparisons On Depths}
\label{sec:full}
From the previous experiment,
we found some evidence that with larger recurrent depth, the performance might improve. To further investigate various implications of depths, 
we carry out a systematic analysis for
both recurrent depth $d_r$ and feedforward depth $d_f$ on \textit{text8} and sequential MNIST datasets.
We build $9$ models in total with
$d_r = 1, 2, 3$ and $d_f = 2, 3, 4$, respectively (as shown in Figure \ref{fig:depth_graph}(b)).
We ensure that all the models have roughly the same number of parameters 
(e.g., the model with $d_r=1$ and $d_f=2$ has a hidden-layer size of $360$).

\begin{table}[h]
\vspace{-10pt}
\caption{Test BPCs of 9 architectures with different $d_r$, $d_f$, $tanh$}
\label{tb:3x3}
\begin{center}
\begin{small}
\begin{sc}
\begin{tabular}{c|@{\hspace{4pt}}c@{\hspace{4pt}}c@{\hspace{4pt}}c}
\hline
\hline
test BPC, $text8$&$d_r = 1$ & $d_r = 2$ & $d_r = 3$\\
\hline
 $d_f = 2$ & $1.88$ & $1.84$ & $\mathbf{1.83}$\\
 $d_f = 3$ & $1.86$ & $1.84$ & $1.85$\\
 $d_f = 4$ & $1.94$ & $1.89$ & $1.88$\\
\hline
\end{tabular}
\end{sc}
\end{small}
\end{center}
\vspace{-5pt}
\end{table}

Table \ref{tb:3x3} and Figure \ref{fig:depth_curve} display results 
on the $text8$ dataset. We observed that for a fixed feedforward deapth $d_f$,
increasing the recurrent depth $d_r$ does improve the model performance,
and the best test BPC is achieved by the architecture
with $d_f = 2$ and $d_r = 3$. 
This suggests that the increase of $d_r$
can aid in better capturing the over-time
nonlinearity of the input sequence. 
However, for a fixed $d_r$, increasing $d_f$ only helps
when $d_r = 1$. For a recurrent  
depth of $d_r = 3$, increasing $d_f$ only hurts models performance. 
This can potentially be attributed to the optimization issues when modelling large input-to-output 
dependencies (see Appendix \ref{sec:full_com} for more details).

With sequential MNIST dataset, we next examined the effects of $d_f$ and $d_r$ 
when modelling long term dependencies (more in Appendix \ref{sec:full_com}).
In particular, we observed that 
increasing $d_f$ does not bring any improvement to the model performance,
and increasing $d_r$ might even be detrimental for training. 
Indeed, it appears that $d_f$ only captures the local nonlinearity
and has less effect on the long term prediction.
This result seems to contradict previous claims~\citep{hermans2013training} that stacked RNNs ($d_f > 1$, $d_r=1$)
could capture information in different time scales and would thus
be more capable of dealing with
learning long-term dependencies. On ther other hand, 
a large $d_r$ indicates multiple transformations per time step, resulting in greater
gradient vanishing/exploding issues~\cite{pascanu2013construct}, which suggests
that $d_r$ should neither be too small nor too large.



\subsection{Recurrent Skip Coefficients}
\label{sec:diff_skip}
To investigate whether increasing a recurrent skip coefficient $s$ improves
model performance on long term dependency tasks, we compare models with
increasing $s$ on the sequential MNIST
problem (without/with permutation, denoted as MNIST and $p$MNIST). 

Our baseline model is the shallow architecture proposed in \citet{le2015simple}. To increase
the recurrent skip coefficient~$s$, we add
connections from time step $t$ to 
time step $t+k$ for some fixed integer $k$, shown in Figure \ref{fig:diff_skip} (a).
By using this specific construction, the recurrent skip
coefficient increases from 1 (i.e., baseline) to $k$. 
We also ensure that each model shares roughly the same number of 
parameters: The hidden layer size  of the baseline model
is set to $90$; Models
with extra connections have
$2$ hidden matrices (one from $t$ to $t+1$ and
the other from $t$ to $t+k$), each with a size of~$64$.
\begin{table}[h]
\vspace{-10pt}
\caption{Test accuracies with different $s$ for $tanh$ RNN}
\label{tb:diff_skip}
\begin{center}
\begin{small}
\begin{sc}
\begin{tabular}{c@{\hspace{2pt}}|@{\hspace{4pt}}c@{\hspace{6pt}}c@{\hspace{6pt}}c@{\hspace{6pt}}c@{\hspace{6pt}}c}
\hline
\hline
 test acc ($\%$)& $s = 1$ & $s = 5$ & $s = 9$ & $s = 13$ & $s = 21$\\
\hline
 MNIST  &$34.9$ &$46.9$ &$74.9$ &$85.4$  & $\mathbf{87.8}$\\
\hline
\hline
  test acc ($\%$)& $s = 1$ & $s = 3$ & $s = 5$ & $s = 7$ &$s = 9$\\
\hline
 $p$MNIST &$49.8$ &$79.1$ &$84.3$ &$\mathbf{88.9}$  &$88.0$\\
\hline
\end{tabular}
\end{sc}
\end{small}
\end{center}
\vspace{-10pt}
\end{table}
\begin{figure}[t]
\includegraphics[width=\columnwidth]{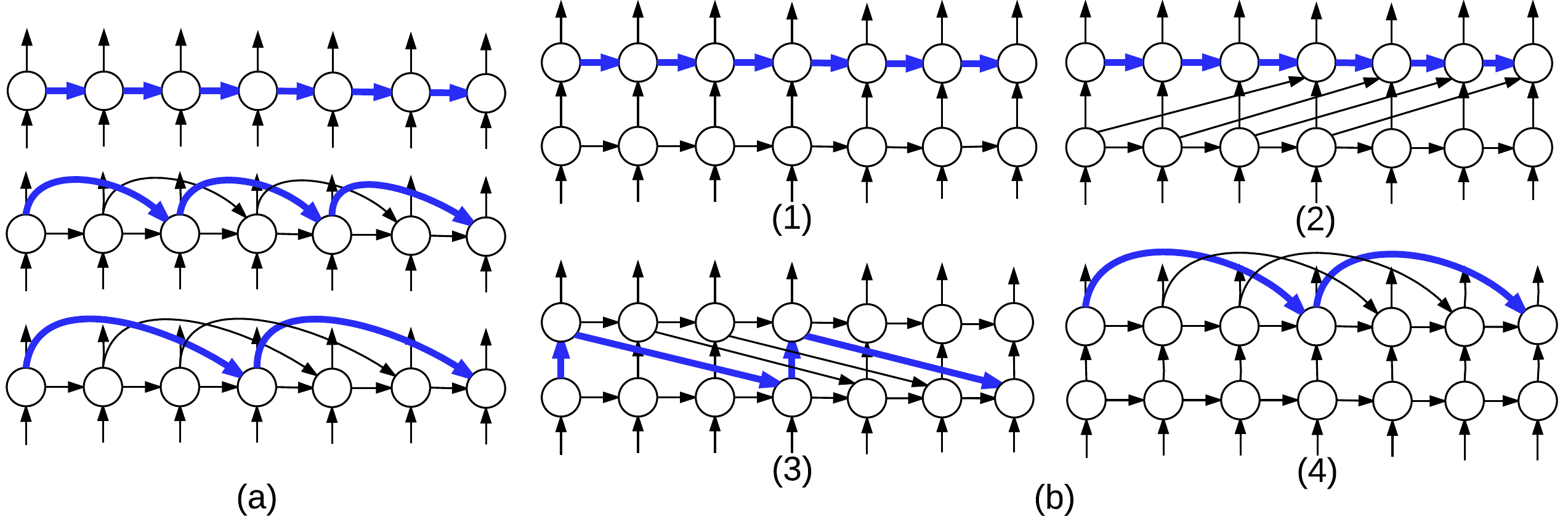}
\vspace{-20pt}
\caption{(a) Proposed architectures in Section \ref{sec:diff_skip}.
From top to bottom are baseline $s=1$, and $s = 2$, $s = 3$. (b) Proposed architectures
in Section \ref{sec:skip_nontrivial} where we take $k = 3$ as an example.
The shortest paths in (a) and (b) that correspond to the
recurrent skip coefficients are colored in blue.}
\label{fig:diff_skip}
\vspace{-10pt}
\end{figure}

The results in Table~\ref{tb:diff_skip} suggest that models with recurrent 
skip coefficient $s$ larger than $1$ could improve the model performance dramatically.
Within a reasonable range of $s$, test accuracy increases quickly as $s$ becomes
larger. We note that our model is the first $tanh$ RNN model
that achieves good performance on this task,
even improving upon the method proposed in \citet{le2015simple}.

In addition, we also formally compare with the previous results reported in 
\citet{le2015simple, arjovsky2015unitary}, where our model (called as sTANH)
has a hidden-layer size of $95$, which is about the same number of parameters in the $tanh$ model of
\citet{arjovsky2015unitary}.
Table \ref{tb:sota} shows 
that our simple architecture
improves upon the state-of-the-art by $2.6\%$ on $p$MNIST,
and achieves almost the same performance as LSTM on the MNIST dataset with only $25\%$ number of parameters  
~\cite{arjovsky2015unitary}.

\begin{figure*}[t]
\includegraphics[width=\textwidth]{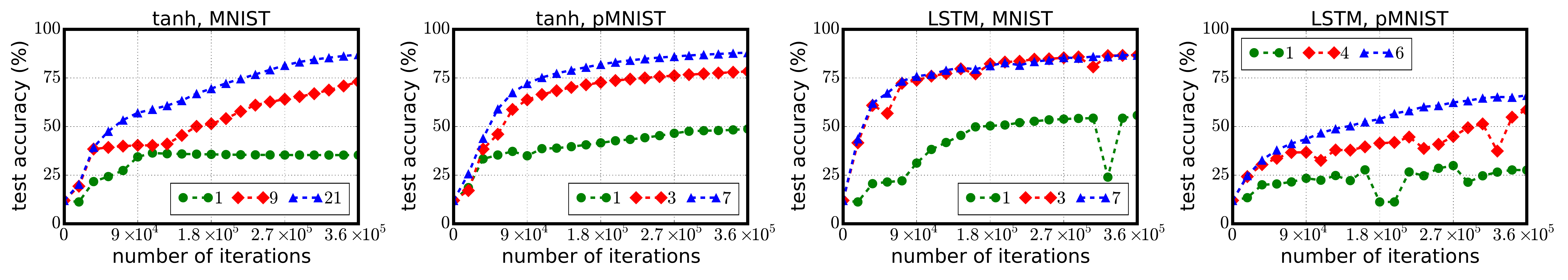}
\vspace{-20pt}
\caption{Test accuracies of $tanh$ and $LSTM$
with different recurrent skip coefficients (numbers in the legend) on MNIST/$p$MNIST.}
\label{fig:mnist_curve}
\vspace{-1pt}
\end{figure*}

\begin{table}[h]
\vspace{-15pt}
\caption{Test accuracies on sequential MNIST and pMNIST tasks}
\label{tb:sota}
\begin{center}
\begin{small}
\begin{sc}
\begin{tabular}{c|ccc}
\hline
\hline
test acc $(\%)$ & MNIST  & $p$MNIST\\
\hline
$i$RNN\cite{le2015simple} &  $97.0$& $\approx82.0$\\
$u$RNN\cite{arjovsky2015unitary}& $95.1$&  $91.4$ \\
LSTM\cite{arjovsky2015unitary} &$\mathbf{98.2}$& $88.0$\\
tanh\cite{le2015simple}& $\approx35.0$& $\approx35.0$\\
\hline
$s$tanh($s = 21, 11$) &$98.1$& $\mathbf{94.0}$\\
\hline
\end{tabular}
\end{sc}
\end{small}
\end{center}
\vspace{-5pt}
\end{table}

Note that obtaining good performance on sequential MNIST requires a larger $s$ than that for
$p$MNIST. We provide an hypothesis for why this happens in Appendix \ref{sec:full_com}.
\begin{figure}[t]
\vspace{-10pt}
\centering
\includegraphics[width=170pt]{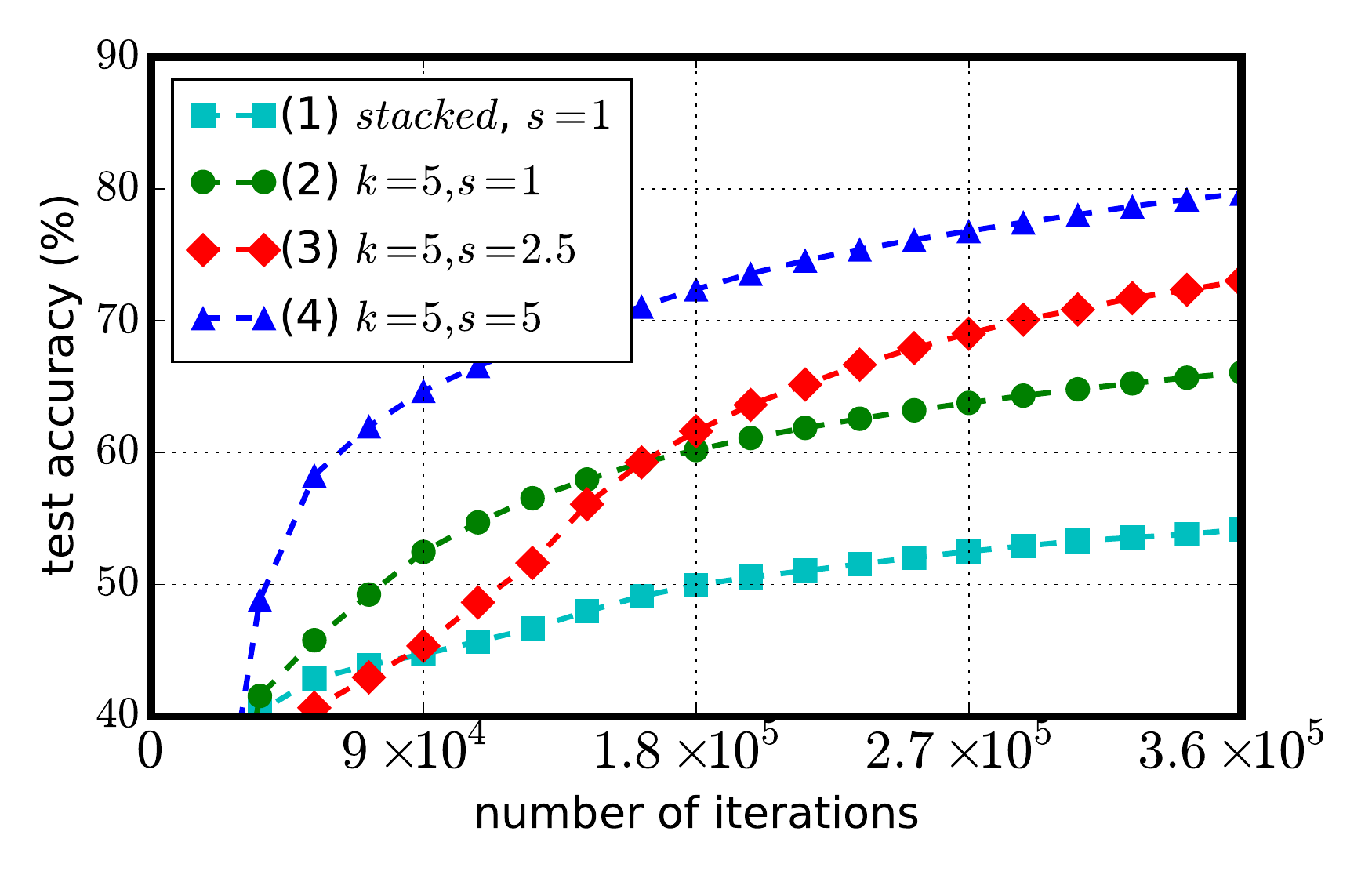}
\vspace{-10pt}
\caption{Test accuracy curves
for architecture (1) (in cyan), (2) (in green), (3) (in red) and (4) (in blue) in Figure \ref{fig:diff_skip} with
the $k = 5$.}
\label{fig:mnist_verify_curve}
\vspace{-10pt}
\end{figure}
\subsection{Recurrent Skip Coefficients vs. Skip Connections}
\label{sec:skip_nontrivial}

In the next set of experiments, we investigated whether the recurrent skip coefficient can suggest something more than simply adding skip connections. We design 4 specific 
architectures shown in Figure \ref{fig:diff_skip}(b): 
(1) is the baseline model with a 2-layer stacked architecture,
the other three models add extra skip connections in different ways.
Note that these extra skip connections
all cross the same time length $k$,
and in particular, (2) and (3) share quite similar architectures.
However, the way in which the skip connections are allocated 
make a big difference on their recurrent skip coefficients: (2) has $s = 1$,
(3) has $s = \frac{k}{2}$ and (4) has $s = k$. We expect that the true measure of a model's capability in long term dependency tasks is its recurrent skip coefficient. Therefore, even though (2), (3) and (4) all add extra skip connections, the fact that their recurrent skip coefficients are different might result in different performance.

\begin{table}[h]
\vspace{-5pt}
\caption{Test accuracies of (1), (2), (3) and (4) for $tanh$ RNN} 
\label{tb:mnist_verify}
\begin{center}
\begin{small}
\begin{sc}
\begin{tabular}{r@{\hspace{2pt}}|@{\hspace{3pt}}c@{\hspace{5pt}}c@{\hspace{5pt}}c@{\hspace{5pt}}c@{\hspace{5pt}}c}
\hline
\hline
test acc ($\%$)\hspace{5pt} &(1)$s = 1$ & (2)$s = 1$& (3)$s = \frac{k}{2}$& (4)$s = k$  \\
\hline     
MNIST\hspace{2pt} $k = 17$ & $39.5$& $39.4$  & $54.2$ & $\mathbf{77.8}$ \\
 $k = 21$ & $39.5$& $39.9$  &$69.6$ &$\mathbf{71.8}$  \\
\hline
$p$MNIST\hspace{5pt} $k = 5$ & $55.5$& $66.6$  & $74.7$ & $\mathbf{81.2}$ \\
\hspace{13pt}$k = 9$ & $55.5$& $71.1$  &$78.6$ &$\mathbf{86.9}$  \\
\hline
\end{tabular}
\end{sc}
\end{small}
\end{center}
\vspace{-5pt}
\end{table}

We evaluated these architectures on the sequential MNIST \&
$p$MNIST datasets. The results show that differences in $s$
indeed cause big performance gaps regardless of the fact that they all have skip connections (see Table~\ref{tb:mnist_verify}
and Figure \ref{fig:mnist_verify_curve}). Given the same $k$, the model with a larger $s$ performs better. In particular, model (3) is better than model (2) even though they only differ in the direction of the skip connections.

It is also very interesting to see that for MNIST (unpermuted), the extra skip connection
in (2) (which does not really increase the recurrent skip coefficient)
brings almost no benefits, as (2) and (1) having almost the same results.
This observation
highlights the following point: when addressing
the long term dependency problems using skip connections,
instead of only considering the time intervals crossed by
the skip connection, one should also consider the model's
recurrent skip coefficient, which can serve as a guide for introducing more powerful skip connections.

\subsection{Results on LSTMs}
\label{sec:exp_LSTM}
Finally, we also performed a similar set of experiments for LSTMs.
For the first experiment comparing 4 architectures $sh$, $st$,
$bu$ and $td$\footnote{$bu$ and $td$ are not well-defined within the conventional understanding of LSTM, where a node can only receive two inputs. We overcome such limitation by considering a Multidimensional LSTM \cite{Graves2007}, which
is explained in detail in Appendix \ref{sec:LSTM}. }, LSTMs have similar performance benefits as $tanh$ units, see Table \ref{tb:butd_lstm}. 
\begin{table}[h]
\vspace{-15pt}
\caption{Test BPCs of \textit{sh}, \textit{st}, \textit{bu}, \textit{td} for LSTM}
\label{tb:butd_lstm}
\begin{center}
\begin{small}
\begin{sc}
\begin{tabular}{c|ccc|c}
\hline
\hline
test BPC&\textit{sh} & \textit{st} & \textit{bu} & \textit{td}\\
\hline
$text8$ &$1.65$ & $1.66$ & $1.65$ & $\mathbf{1.63}$ \\
\hline
\end{tabular}
\end{sc}
\end{small}
\end{center}
\vspace{-10pt}
\end{table}


LSTMs also showed performance boosts and much faster
convergence speed when using larger $s$, as displayed in Table~\ref{tb:diff_skip_lstm}.
In sequential MNIST, the LSTM with $s=3$ already
performs quite well and increasing $s$ did not result in any significant improvement,
while in $p$MNIST, the performance gradually improves  as $s$ increases
from $4$ to $6$. We also observed that the LSTM performed worse on permuted MNIST
compared to a $tanh$ RNN, which was also reported in~\citet{le2015simple}.

\begin{table}[h]
\vspace{-10pt}
\caption{Test accuracies with different $s$ for LSTMs}
\label{tb:diff_skip_lstm}
\begin{center}
\begin{small}
\begin{sc}
\begin{tabular}{c@{\hspace{2pt}}|@{\hspace{4pt}}c@{\hspace{6pt}}c@{\hspace{6pt}}c@{\hspace{6pt}}c@{\hspace{6pt}}c}
\hline
\hline
test acc ($\%$)    & $s = 1$ & $s = 3$ & $s = 5$ & $s = 7$ & $s = 9$\\
\hline
  MNIST&$56.2$ &$\mathbf{87.2}$ &$86.4$ &$86.4$  & $84.8$\\
\hline
\hline
test acc ($\%$)  & $s = 1$ & $s = 3$ & $s = 4$ & $s = 5$ &$s = 6$\\
\hline
 $p$MNIST  &$28.5$ &$25.0$ &$60.8$ &$62.2$  &$\mathbf{65.9}$\\
\hline
\end{tabular}
\end{sc}
\end{small}
\end{center}
\vspace{-5pt}
\end{table}

To conclude, the empirical evidence indicates 
that LSTMs might also benefit from increasing recurrent depth on some tasks and 
from increasing recurrent skip coefficients on long term dependency problems.

\section{Conclusion}
In this paper, we first introduced a general definition of RNNs, which allows one to construct more general architectures, and provides a solid framework for the architectural complexity analysis. We then proposed three architectural complexity measures: recurrent depth, feedforward depth, and recurrent skip coefficients, each capturing the complexity in the long term, complexity in the short term and the speed of information flow. We also find empirical evidence that increasing depths might yield performance improvements, increasing feedforward depth might not help on long
term dependency tasks, while increasing the recurrent skip coefficient
can largely improves the performance on long term dependency tasks. 
These measures and results can provide guidance for the design of new 
recurrent architectures for a particular learning task.
Future work could involve more comprehensive studies (e.g., on more datasets, using more architectures with various transition functions) to investigate the effectiveness of the proposed measures. 
\section*{Acknowledgments}
The authors acknowledge the following agencies for funding and support: NSERC, Canada Research Chairs, CIFAR, Calcul Quebec, Compute Canada, Samsung and Disney Research. 
The authors thank the developers of Theano \cite{Bastien-Theano-2012} and Keras \cite{chollet2015}, and also thank Nicolas Ballas,
Tim Cooijmans, Ryan Lowe and Mohammad Pezeshki for their insightful comments.

%
%


{\small 
\bibliography{ourpaper}
\bibliographystyle{icml2016}
}
